\documentclass[11pt]{article}
\usepackage[T1]{fontenc}
\usepackage[utf8]{inputenc}
\IfFileExists{lmodern.sty}{\usepackage{lmodern}}{}
\usepackage{amsmath,amssymb,amsthm,mathtools}
\usepackage[numbers,sort&compress]{natbib}
\usepackage[margin=1in]{geometry}
\usepackage{booktabs}
\usepackage[hidelinks]{hyperref}
\usepackage{xurl}

\theoremstyle{plain}
\newtheorem{theorem}{Theorem}[section]
\newtheorem{proposition}[theorem]{Proposition}
\newtheorem{lemma}[theorem]{Lemma}
\newtheorem{corollary}[theorem]{Corollary}
\theoremstyle{definition}
\newtheorem{definition}[theorem]{Definition}

\theoremstyle{remark}
\newtheorem{remark}[theorem]{Remark}

\newcommand{\pos}[1]{\left[#1\right]_+}
\newcommand{\HD}{H_D'}
\newcommand{\HC}{H_C'}
\newcommand{\KCDI}{\mathrm{KCDI}}
\newcommand{\KJI}{\mathrm{KJI}}
\newcommand{\Abar}{\bar{A}}
\newcommand{\dtv}{d_{\mathrm{TV}}}
\newcommand{\E}{\mathbb{E}}
\newcommand{\Var}{\operatorname{Var}}
\newcommand{\pstar}{p^{*}}

\title{Bounds, Decompositions and Null Behaviour of KRATOS:\\
A Mathematical Specification of a Recognition-Comparability Diagnostic}

\author{%
María Dolores González Barbado\thanks{Corresponding author: \texttt{mgonzalezbarbado@uoc.edu}. ORCID: 0000-0002-4213-2090.}\\
\small Interuniversity Doctoral Programme in Tourism,\\
\small Universitat Oberta de Catalunya (UOC), Barcelona, Spain
\and
Alberto Barbado González\thanks{ORCID: 0000-0002-5123-5243.}\\
\small Escuela Superior de Ingeniería y Tecnología,\\
\small Universidad Internacional de La Rioja (UNIR), Logroño, Spain}
\date{}

\begin{document}
\maketitle

\begin{abstract}
KRATOS is a group-structured bibliometric diagnostic that compares the distribution of documents (participation) with the distribution of citation weight (recognition) over a fixed, finite universe of analytical groups. This note gives its complete mathematical specification and derives the properties that govern its interpretation. Beyond bounds and equality conditions for each component, we show that the recognition-alignment score depends on citation data only through ratios of group mean citation rates to the corpus mean; that the composition ceiling of the participation--recognition factor is an affine function of the total variation distance to the uniform reference; and that the factor itself satisfies Fréchet-type bounds in terms of this ceiling and a participation-weighted recognition score. We establish an exact logarithmic decomposition of the composite index, an ordering between the primary and a reciprocal-symmetric recognition score, invariance properties, and closed-form first and second moments of the recognition ratios under global and stratified permutation nulls. These results explain why composite orderings can be sensitive to small groups, heavy-tailed citation distributions and metadata reassignment. All results are verified numerically with an accompanying script. The specification concerns measurement structure only; it does not define a measure of epistemic change or justice.
\end{abstract}



% =====================================================================
\section{Introduction}\label{sec:intro}
% =====================================================================

Comparing citation-based recognition across predefined groups requires two distributions over the same finite set of groups: the share of documents each group contributes and the share of citation weight those documents receive. KRATOS (Knowledge Representation and Analysis for Transdisciplinary Oversight of Science) maps this pair of distributions to a small set of scalar quantities: two evenness components, their weighted geometric combination (KCDI), a participation--recognition factor $P$, and the composite $\KJI=\KCDI\times P$. Its empirical application to four bibliographic corpora, together with conditional permutation calibration and sensitivity analyses, is reported in a companion manuscript \citep{gonzalezbarbado2026kratos}; the frozen implementation is archived at \citet{kratos2026zenodo}.

The purpose of this note is narrower and formal. Composite indicators are routinely reported without a statement of what their components can and cannot jointly attain, how they respond to changes in their inputs, or how they behave under the resampling schemes used to calibrate them \citep{oecd2008composite}. For KRATOS these questions are not cosmetic: the companion study finds that cross-corpus orderings of KJI change under metadata reassignment and resampling, and an explanation of that sensitivity has to come from the structure of the index itself.

The note makes four contributions. First, it fixes notation and boundary conventions so that the mathematical object and its software implementation coincide (Sections~\ref{sec:notation} and~\ref{sec:computation}). Second, it characterises the participation--recognition factor: the recognition score is a function of relative mean citation rates (Proposition~\ref{prop:rate}); the composition ceiling is an affine transform of a total variation distance (Proposition~\ref{prop:abar-tv}); and the factor is bounded above and below by a ceiling term and a recognition term (Theorem~\ref{thm:frechet}). Third, it gives an exact logarithmic decomposition of KJI and an ordering between alternative recognition scores (Section~\ref{sec:kji}). Fourth, it derives closed-form null moments under global and stratified permutation (Section~\ref{sec:null}), which make explicit why null dispersion grows for small groups and heavy-tailed citation weights.

None of the individual building blocks is new. Normalised Shannon entropy is Pielou's evenness \citep{shannon1948communication,pielou1966measurement}, its use for diversity in science is well established \citep{stirling2007diversity}, and the moment formulas rest on classical sampling theory \citep{cochran1977sampling}. The contribution lies in the joint specification and in the properties of the composite that follow from it.

% =====================================================================
\section{Domain and notation}\label{sec:notation}
% =====================================================================

Let $D$ be a canonical bibliographic corpus with $N=|D|$ documents. Each document is either assigned to exactly one group of a fixed analytical universe $U=\{1,\dots,G\}$, $G\geq 2$, or retained in an \emph{audit state} when the metadata needed for assignment are unresolved. Let $D_R\subseteq D$ be the resolved set, $N_R=|D_R|$, and define coverage $\kappa=N_R/N$ for $N>0$. Audit states are not additional groups: they never enter $U$, and every quantity below is computed on $D_R$ only. Groups of $U$ with no resolved documents remain in $U$.

For $u\in U$, let $n_u$ be the number of resolved documents assigned to $u$. Each resolved document $d_i$ carries a citation weight $x_i\geq 0$ (raw citation counts in the primary specification). Define
\[
c_u=\sum_{i:\,d_i\in u}x_i,\qquad C_R=\sum_{u\in U}c_u .
\]
For $N_R>0$ and $C_R>0$, the participation and recognition vectors are
\[
p_u=\frac{n_u}{N_R},\qquad s_u=\frac{c_u}{C_R},\qquad p,s\in\Delta_G=\Bigl\{z\in\mathbb{R}^G_{\geq 0}:\textstyle\sum_u z_u=1\Bigr\}.
\]
Results stated for $p,s\in\Delta_G$ hold on the whole simplex; in applications $p$ lies on the lattice $\{n/N_R\}$, whereas any $s\in\Delta_G$ is realisable by a suitable allocation of non-negative weights. The participation reference is the uniform share $\pstar=1/G$ and $e=(\pstar,\dots,\pstar)$ denotes the uniform vector. For $z,z'\in\Delta_G$, $\dtv(z,z')=\tfrac12\sum_u|z_u-z'_u|$ is the total variation distance. Finally, for $n_u>0$ we write $\bar x_u=c_u/n_u$ for the mean citation weight of group $u$ and $\bar x=C_R/N_R$ for the corpus mean, and $[a]_+=\max(0,a)$.

% =====================================================================
\section{Distributional evenness}\label{sec:evenness}
% =====================================================================

\begin{definition}[Evenness]\label{def:evenness}
For $z\in\Delta_G$, let $H'(z)=-(\ln G)^{-1}\sum_{u\in U}z_u\ln z_u$ with $0\ln 0=0$. Document- and citation-distribution evenness are $\HD=H'(p)$ and, for $C_R>0$, $\HC=H'(s)$.
\end{definition}

$H'$ is Pielou's evenness of a distribution over $G$ categories \citep{pielou1966measurement}.

\begin{proposition}[Range and extremal configurations]\label{prop:entropy}
For every $z\in\Delta_G$, $0\leq H'(z)\leq 1$. The value $0$ is attained exactly at the vertices of $\Delta_G$ and the value $1$ exactly at $e$. $H'$ is invariant under permutations of the coordinates.
\end{proposition}

\begin{proof}
Each term $-z_u\ln z_u$ is non-negative on $[0,1]$ and vanishes only at $z_u\in\{0,1\}$; on $\Delta_G$ all terms vanish exactly at a vertex. Shannon entropy is strictly concave and symmetric on $\Delta_G$ and attains its unique maximum $\ln G$ at $e$ \citep{cover2006elements}. Dividing by $\ln G>0$ gives the claim.
\end{proof}

\begin{remark}\label{rem:G}
The normalisation fixes the meaning of $H'$ relative to $G$. Values computed over universes of different size are not comparable merely because both lie in $[0,1]$, and empty groups affect $H'$ through $G$. Changing $U$ defines a different measurement specification.
\end{remark}

% =====================================================================
\section{Knowledge Contribution Diversity Index}\label{sec:kcdi}
% =====================================================================

\begin{definition}[KCDI]\label{def:kcdi}
For $\lambda\in[0,1]$,
\[
\KCDI_\lambda=
\begin{cases}
\HC, & \lambda=0,\\
(\HD)^{\lambda}(\HC)^{1-\lambda}, & 0<\lambda<1,\\
\HD, & \lambda=1.
\end{cases}
\]
The primary specification uses $\lambda=1/2$; $\lambda$ is fixed analytically and not estimated.
\end{definition}

\begin{proposition}[Bounds, monotonicity and limited compensability]\label{prop:kcdi}
For all $\HD,\HC\in[0,1]$ and $\lambda\in[0,1]$:
\begin{enumerate}
\item $\min(\HD,\HC)\leq\KCDI_\lambda\leq\lambda\HD+(1-\lambda)\HC\leq\max(\HD,\HC)\leq 1$;
\item for $0<\lambda<1$, $\KCDI_\lambda$ is non-decreasing in each argument and equals zero whenever either argument is zero;
\item if $\HD,\HC>0$, then
\[
\frac{\partial\ln\KCDI_\lambda}{\partial\ln\HD}=\lambda,\qquad
\frac{\partial\ln\KCDI_\lambda}{\partial\ln\HC}=1-\lambda,\qquad
\frac{\partial\ln\KCDI_\lambda}{\partial\lambda}=\ln\frac{\HD}{\HC}.
\]
\end{enumerate}
\end{proposition}

\begin{proof}
At the endpoints all claims are immediate. For $0<\lambda<1$ the middle inequality in (i) is the weighted arithmetic--geometric mean inequality, and the outer inequalities hold because a weighted geometric mean of two numbers lies between them. Part (ii) follows from monotonicity of $t\mapsto t^\alpha$ for $\alpha>0$ and from $0^\alpha=0$. Part (iii) follows from $\ln\KCDI_\lambda=\lambda\ln\HD+(1-\lambda)\ln\HC$.
\end{proof}

Part (i) states formally the non-compensatory character of geometric aggregation \citep{oecd2008composite}: KCDI never exceeds the corresponding arithmetic mean, and the gap widens as the components diverge. By (iii), raising $\lambda$ increases KCDI if and only if $\HD>\HC$.

\begin{proposition}[Endpoint continuity]\label{prop:continuity}
$\lambda\mapsto\KCDI_\lambda$ is right-continuous at $\lambda=0$ if and only if $\HD>0$ or $\HC=0$, and left-continuous at $\lambda=1$ if and only if $\HC>0$ or $\HD=0$.
\end{proposition}

\begin{proof}
If $\HD>0$, $(\HD)^\lambda(\HC)^{1-\lambda}\to\HC$ as $\lambda\downarrow 0$, including when $\HC=0$. If $\HD=0$, the interior value is $0$ for all $\lambda>0$, so the right limit is $0$, which equals $\HC$ only if $\HC=0$. The case $\lambda\uparrow 1$ is symmetric.
\end{proof}

% =====================================================================
\section{Participation--recognition alignment}\label{sec:alignment}
% =====================================================================

\begin{definition}[Alignment scores and derived factors]\label{def:alignment}
For $u\in U$ define participation alignment and, for $C_R>0$, the recognition ratio and recognition alignment
\[
A_u=\pos{1-\frac{|p_u-\pstar|}{\pstar}}=\pos{1-|Gp_u-1|},\qquad
r_u=\frac{s_u}{p_u},\qquad
S_u=\begin{cases}\pos{1-|r_u-1|}, & p_u>0,\\ 0, & p_u=0.\end{cases}
\]
If $C_R=0$, set $S_u=0$ for all $u$. The participation--recognition factor, the composition ceiling, the participation-weighted recognition score and, for $\Abar>0$, the conditional recognition alignment are
\[
P=\frac1G\sum_{u\in U}A_uS_u,\qquad
\Abar=\frac1G\sum_{u\in U}A_u,\qquad
P_S=\sum_{u\in U}p_uS_u,\qquad
Q=\frac{P}{\Abar}=\sum_{u\in U}\frac{A_u}{\sum_v A_v}\,S_u .
\]
The complement $R=1-P$ is reported by the software for compatibility and has no separate role.
\end{definition}

$Q$ is the $A$-weighted mean of the recognition scores, so $P=\Abar\,Q$ separates a pure participation term from recognition alignment conditional on participation.

\begin{proposition}[Component scores]\label{prop:component}
For every $u$, $A_u,S_u\in[0,1]$. Moreover $A_u=1\iff p_u=\pstar$, $A_u>0\iff 0<p_u<2/G$, and for $p_u>0$, $C_R>0$,
\[
S_u=\pos{\min(r_u,\,2-r_u)}=
\begin{cases} r_u, & 0\leq r_u\leq 1,\\ 2-r_u, & 1\leq r_u\leq 2,\\ 0, & r_u\geq 2,\end{cases}
\]
so that $S_u=1\iff s_u=p_u$.
\end{proposition}

\begin{proof}
All statements follow from the positive-part definitions. $A_u>0$ iff $|Gp_u-1|<1$, i.e.\ $0<p_u<2/G$. For $r\geq 0$, $1-|r-1|$ equals $r$ on $[0,1]$, $2-r$ on $[1,2]$ and is non-positive for $r\geq 2$.
\end{proof}

\begin{proposition}[Relative mean-rate form]\label{prop:rate}
If $C_R>0$ and $n_u>0$, then $r_u=\bar x_u/\bar x$. Consequently $S_u$ depends on the citation weights only through the ratio of the group mean to the corpus mean, and $\sum_{u}p_u r_u=1$.
\end{proposition}

\begin{proof}
$r_u=(c_u/C_R)/(n_u/N_R)=(c_u/n_u)/(C_R/N_R)$. The identity follows from $\sum_u p_ur_u=\sum_u s_u=1$.
\end{proof}

\begin{remark}\label{rem:rate}
Proposition~\ref{prop:rate} makes the main source of fragility of $S_u$ explicit. Recognition alignment is a parity statement about group \emph{means}. In a group with few documents, one highly cited document can move $\bar x_u$ above $2\bar x$ and set $S_u=0$, irrespective of how citations are distributed among the remaining documents.
\end{remark}

\begin{lemma}\label{lem:AGp}
For every $u$, $A_u\leq Gp_u$, that is, $A_u/G\leq p_u$.
\end{lemma}

\begin{proof}
If $Gp_u\leq 1$ then $A_u=Gp_u$. If $Gp_u>1$ then $A_u\leq 1<Gp_u$.
\end{proof}

\begin{proposition}[Composition ceiling and total variation]\label{prop:abar-tv}
$\Abar\geq 1-2\,\dtv(p,e)$, with equality if and only if $\max_u p_u\leq 2/G$. In particular, $\Abar=1\iff p=e$.
\end{proposition}

\begin{proof}
Since $[a]_+\geq a$, $\Abar\geq G^{-1}\sum_u(1-|Gp_u-1|)=1-\sum_u|p_u-\pstar|=1-2\,\dtv(p,e)$. Equality holds term by term iff $1-|Gp_u-1|\geq 0$ for every $u$, i.e.\ $p_u\leq 2/G$. The last claim follows from $\Abar=1\iff A_u=1$ for all $u$.
\end{proof}

\begin{proposition}[Participation-weighted recognition and total variation]\label{prop:ps-tv}
If $C_R>0$, then $P_S\geq 1-2\,\dtv(p,s)$, with equality if and only if $r_u\leq 2$ for every $u$ with $p_u>0$ and $s_u=0$ for every $u$ with $p_u=0$.
\end{proposition}

\begin{proof}
For $p_u>0$, $p_u(1-S_u)=p_u\min(1,|r_u-1|)\leq p_u|r_u-1|=|s_u-p_u|$, with equality iff $r_u\leq 2$. Hence $1-P_S=\sum_{p_u>0}p_u(1-S_u)\leq\sum_{p_u>0}|s_u-p_u|\leq\sum_u|s_u-p_u|=2\,\dtv(p,s)$; the last inequality is an equality iff $s_u=0$ whenever $p_u=0$.
\end{proof}

\begin{theorem}[Range of $P$ for fixed participation]\label{thm:P}
For every $p,s\in\Delta_G$, $0\leq P\leq\Abar\leq 1$, and $P=\Abar$ if and only if $S_u=1$ for every $u$ with $A_u>0$. For fixed $p$, both bounds are attained: $P=\Abar$ at $s=p$, and $P=0$ at some vertex $s$ of $\Delta_G$.
\end{theorem}

\begin{proof}
The chain follows from $0\leq S_u\leq 1$ and $A_u\leq 1$, and $\Abar-P=G^{-1}\sum_uA_u(1-S_u)$ is a sum of non-negative terms that vanishes iff each term with $A_u>0$ vanishes. At $s=p$, $S_u=1$ whenever $p_u>0$, while $A_u=0$ when $p_u=0$. If $p$ is a vertex, every $A_u=0$ and $P=0$ for all $s$. Otherwise $p$ has at least two positive coordinates, so some $j$ satisfies $0<p_j\leq 1/2$. Taking $s$ to be the vertex at $j$ gives $r_j=1/p_j\geq 2$ and $r_u=0$ for $u\neq j$, so every $S_u=0$ and $P=0$.
\end{proof}

\begin{theorem}[Fréchet-type bounds]\label{thm:frechet}
For every $p,s\in\Delta_G$,
\[
\max\bigl(0,\ \Abar+P_S-1\bigr)\ \leq\ P\ \leq\ \min\bigl(\Abar,\ P_S\bigr).
\]
\end{theorem}

\begin{proof}
By Lemma~\ref{lem:AGp}, $P=\sum_u(A_u/G)S_u\leq\sum_up_uS_u=P_S$ and
$\Abar-P=\sum_u(A_u/G)(1-S_u)\leq\sum_up_u(1-S_u)=1-P_S$.
Together with Theorem~\ref{thm:P} this gives both bounds.
\end{proof}

\begin{corollary}\label{cor:tv}
If $C_R>0$, then $P\geq\Abar-2\,\dtv(p,s)$ and $\Abar-P\leq 2\,\dtv(p,s)$. Moreover, $P\geq 1-2\,\dtv(p,e)-2\,\dtv(p,s)$.
\end{corollary}

\begin{proof}
Combine Theorem~\ref{thm:frechet} with Propositions~\ref{prop:ps-tv} and~\ref{prop:abar-tv}.
\end{proof}

Theorem~\ref{thm:frechet} and Corollary~\ref{cor:tv} give $P$ a two-distance reading: it is close to $\Abar$ when recognition is close to participation in total variation, and $\Abar$ itself measures the distance of participation from the reference.

\begin{proposition}[Evenness and the ceiling are not interchangeable]\label{prop:nonmonotone}
For $G\geq 4$, $\HD$ and $\Abar$ are not monotone functions of each other on $\Delta_G$.
\end{proposition}

\begin{proof}
For $G=4$, $p=(0.5,0.5,0,0)$ gives $\HD=0.5$ and $\Abar=0$, whereas $p'=(0.97,0.01,0.01,0.01)$ gives $\HD\approx 0.121$ and $\Abar=0.03$; thus $\HD(p)>\HD(p')$ but $\Abar(p)<\Abar(p')$. For general $G\geq 4$, let $p=(\tfrac12,\tfrac12,0,\dots,0)$, so that $\HD(p)=\ln 2/\ln G>0$ and, since $\tfrac12\geq 2/G$, $\Abar(p)=0$. Let $p'_\varepsilon=(1-(G-1)\varepsilon,\varepsilon,\dots,\varepsilon)$ with $0<\varepsilon<1/G$. Then $\Abar(p'_\varepsilon)\geq(G-1)\varepsilon>0$, while $\HD(p'_\varepsilon)\to 0$ as $\varepsilon\downarrow 0$, so the same reversal holds for small $\varepsilon$.
\end{proof}

The two participation-based quantities therefore carry different information. $\HD$ summarises the whole distribution, whereas $\Abar$ ignores every group that holds at least twice the reference share. Both should be reported.

\begin{definition}[Reciprocal-symmetric variant]\label{def:slog}
Set $S^{\log}_u=\pos{1-|\ln r_u|/\ln 2}$ if $p_u,s_u>0$, and $S^{\log}_u=0$ otherwise. Define
\[
P^{\log}=\frac1G\sum_uA_uS^{\log}_u,\qquad \KJI^{\log}_\lambda=\KCDI_\lambda\,P^{\log}.
\]
\end{definition}

$S^{\log}$ satisfies $S^{\log}(r)=S^{\log}(1/r)$, whereas the primary score gives $S(1/2)=1/2$ and $S(2)=0$. The variant is used only for sensitivity analysis.

\begin{proposition}[Ordering of recognition scores]\label{prop:slog}
For every $u$, $S^{\log}_u\leq S_u$. Consequently $P^{\log}\leq P$ and $\KJI^{\log}_\lambda\leq\KJI_\lambda$.
\end{proposition}

\begin{proof}
If $p_u=0$, $s_u=0$ or $r_u\geq 2$, then $S^{\log}_u=0\leq S_u$; the same holds on $(0,1/2]$, where $S^{\log}_u=0$. On $[1,2)$, concavity of $\log_2$ with equality at $r=1,2$ gives $\log_2 r\geq r-1$, so $S^{\log}_u=1-\log_2r_u\leq 2-r_u=S_u$. On $(1/2,1]$, $f(r)=r-1-\log_2r$ satisfies $f(1)=0$ and $f'(r)=1-1/(r\ln 2)<0$, so $f\geq 0$ and $S^{\log}_u=1+\log_2r_u\leq r_u=S_u$. Multiplying by $A_u\geq 0$ and $\KCDI_\lambda\geq 0$ gives the remaining claims.
\end{proof}

Switching to the reciprocal-symmetric form can therefore only lower $P$ and KJI. Any reordering it produces across corpora reflects differences in how far each corpus sits from $r_u=1$, not a change of sign.

% =====================================================================
\section{The composite index}\label{sec:kji}
% =====================================================================

\begin{definition}[KJI]\label{def:kji}
$\KJI_\lambda=\KCDI_\lambda\times P$. The word \emph{Justice} in the name is a legacy label of the software and carries no normative meaning.
\end{definition}

\begin{theorem}[Bounds and maximum]\label{thm:kji}
For all admissible inputs,
\[
0\leq\KJI_\lambda\leq\KCDI_\lambda\min(\Abar,P_S)\leq\KCDI_\lambda\leq 1,
\qquad
\KCDI_\lambda-\KJI_\lambda=\KCDI_\lambda(1-P).
\]
For $N_R>0$ and $C_R>0$, $\KJI_\lambda=1$ if and only if $p=s=e$.
\end{theorem}

\begin{proof}
The bounds follow from Theorem~\ref{thm:frechet} and Proposition~\ref{prop:kcdi}; the identity is a rearrangement of the definition. If $p=s=e$, all components equal one. Conversely, $\KJI_\lambda=1$ requires $P=1$; since $P\leq\Abar\leq1$, this forces $A_u=1$ for all $u$, hence $p=e$ (Proposition~\ref{prop:abar-tv}), and $S_u=1$ for all $u$, hence $s=p=e$.
\end{proof}

\begin{proposition}[Exact logarithmic decomposition]\label{prop:logdec}
Let $0<\lambda<1$ and consider two configurations, indexed $0$ and $1$, with $\HD,\HC,P>0$ in both. Then
\[
\Delta\ln\KJI_\lambda=\lambda\,\Delta\ln\HD+(1-\lambda)\,\Delta\ln\HC+\Delta\ln\Abar+\Delta\ln Q,
\]
where $\Delta f=f^{(1)}-f^{(0)}$ and $Q=P/\Abar$.
\end{proposition}

\begin{proof}
Take logarithms of $\KJI_\lambda=(\HD)^\lambda(\HC)^{1-\lambda}\Abar Q$ in each configuration and subtract. $P>0$ implies $\Abar>0$, so $Q$ is defined.
\end{proof}

The decomposition is an accounting identity. Its terms are not statistically independent: $\HD$ and $\Abar$ are both functions of $p$, while $\HC$ and $Q$ depend on $s$ given $p$. It allocates a change in KJI, for example under an alternative assignment of audit-state records, exactly among evenness, composition and conditional recognition. It does not attribute causes.

\begin{proposition}[Invariances]\label{prop:invariance}
\begin{enumerate}
\item \emph{Scale.} Replacing $x_i$ by $\alpha x_i$ for all $i$, with $\alpha>0$, leaves $s$, and hence $\HC$, $S_u$, $P$, $P_S$, $\KCDI_\lambda$ and $\KJI_\lambda$, unchanged.
\item \emph{Relabelling.} For any permutation $\pi$ of $U$, applying $\pi$ to the group labels permutes $(A_u)$, $(S_u)$ and $(r_u)$ accordingly and leaves $\HD,\HC,\Abar,P,P_S,\KCDI_\lambda,\KJI_\lambda$ unchanged.
\item \emph{Audit states.} Adding or removing documents in audit states changes $\kappa$ but no other quantity.
\end{enumerate}
\end{proposition}

\begin{proof}
(i) $s_u=\alpha c_u/(\alpha C_R)$. (ii) All scalar quantities are symmetric sums over $U$ and $\pstar$ is uniform. (iii) Audit-state documents do not enter $D_R$.
\end{proof}

Invariance (i) does not extend to group-, year- or document-specific rescalings. The year-normalised weights $\tilde x_i=x_i/\bar x_{y(i)}$ used in the companion study rescale each publication cohort by a different constant and therefore change $s$. Invariance (iii) is why coverage $\kappa$ must always be reported alongside the indices: they are blind to how much of the corpus is resolved.

% =====================================================================
\section{Behaviour under permutation nulls}\label{sec:null}
% =====================================================================

Calibration in the companion study permutes citation weights among resolved documents while holding group labels fixed, either globally or within temporal strata \citep{ernst2004permutation}. For an observed statistic $T_{\mathrm{obs}}$ and $B$ sampled permutations with values $T_b$, the lower-tail Monte Carlo probability is $\hat p_{\mathrm{MC},L}=(1+\#\{b:T_b\leq T_{\mathrm{obs}}\})/(B+1)$ \citep{phipson2010permutation}. The following results describe the null distribution of the inputs to $P$.

\begin{proposition}[Global permutation]\label{prop:null-global}
Let the multiset $\{x_i\}_{i\in D_R}$ have mean $\bar x>0$ and population variance $\sigma_x^2$, and let citation weights be assigned to resolved documents by a uniformly random permutation with group labels fixed. Then $C_R$, $p$, $\HD$, $(A_u)$ and $\Abar$ are invariant, and for each $u$ with $n_u>0$,
\[
\E[s_u]=p_u,\qquad \E[r_u]=1,\qquad
\Var(r_u)=\frac{N_R-n_u}{n_u(N_R-1)}\,\frac{\sigma_x^2}{\bar x^2}.
\]
\end{proposition}

\begin{proof}
Under the permutation, $c_u$ is the total of a simple random sample of size $n_u$ drawn without replacement from a finite population of size $N_R$ with mean $\bar x$ and variance $\sigma_x^2$. Hence $\E[c_u]=n_u\bar x$ and $\Var(c_u)=n_u\sigma_x^2(N_R-n_u)/(N_R-1)$ \citep{cochran1977sampling}. Because the multiset is preserved, $C_R=N_R\bar x$ is constant, and $r_u=c_u/(n_u\bar x)$ yields the stated moments. Quantities depending only on group labels are unchanged.
\end{proof}

\begin{proposition}[Stratified permutation]\label{prop:null-strat}
Let resolved documents be partitioned into strata $h=1,\dots,K$ of sizes $n_h$, and let permutations act independently within each stratum. Write $n_{hu}$ for the number of documents of group $u$ in stratum $h$, and $\bar x_h$, $\sigma_h^2$ for the within-stratum mean and population variance of citation weights. Then $C_R$ and $(A_u)$ are invariant and
\[
\E[c_u]=\sum_h n_{hu}\bar x_h,\qquad
\Var(c_u)=\sum_{h:\,n_h>1}\frac{n_{hu}(n_h-n_{hu})}{n_h-1}\,\sigma_h^2 .
\]
In general $\E[s_u]\neq p_u$: the null centre of $s_u$ is the share that group $u$ would receive if each of its documents earned its stratum mean.
\end{proposition}

\begin{proof}
Apply Proposition~\ref{prop:null-global} within each stratum and sum over independent strata. A stratum with $n_h=1$ contributes no variance.
\end{proof}

\begin{remark}[Interpretation for calibration]\label{rem:null}
Three consequences follow. First, permutation calibration of $P$ is conditional on $\Abar$: only the recognition branch varies under the null. Second, null dispersion of $r_u$, and therefore of $S_u$, grows with the squared coefficient of variation $\sigma_x^2/\bar x^2$ and approximately with $1/n_u$. Heavy-tailed citation distributions and small groups thus produce wide nulls in which departures of fixed size are harder to resolve. Third, stratification removes between-stratum citation variation from the null and recentres $s_u$ on the stratum-adjusted expectation. With few documents per stratum the attainable variation is small, and strata with $n_h=1$ or with a single group contribute none. The structural size of the between-group reassignment space is summarised by
\[
\log\Omega_{\mathrm{group}}=\sum_h\Bigl[\log(n_h!)-\sum_u\log(n_{hu}!)\Bigr],
\]
which counts label-distinct assignments. Because citation weights repeat, it overstates the number of distinct values of the statistics.
\end{remark}

% =====================================================================
\section{Boundary conventions and computation}\label{sec:computation}
% =====================================================================

The mathematical object and its implementation coincide under the following conventions.
\begin{enumerate}
\item If $N_R=0$: $\kappa=0$ and $\HD=\HC=\KCDI_\lambda=P=\Abar=P_S=\KJI_\lambda=0$; shares are not evaluated.
\item If $N_R>0$ and $C_R=0$: compute $p$, $\HD$, $(A_u)$ and $\Abar$ normally; set $\HC=0$ and $S_u=0$ for all $u$, so that $P=P_S=0$. KCDI follows Definition~\ref{def:kcdi}, and is zero for interior $\lambda$. These zeros encode the absence of citation information, not maximal concentration.
\item In all cases, $0\ln 0=0$; $S_u=0$ for $p_u=0$ without evaluating $s_u/p_u$; and the endpoint cases of $\lambda$ are evaluated directly, never as $0^0$.
\end{enumerate}

\noindent\textbf{Algorithm.} Input: labels $g_i\in U\cup\{\text{audit}\}$ for $d_i\in D$, weights $x_i\geq 0$, $\lambda\in[0,1]$.
\begin{enumerate}
\item Restrict to $D_R$; compute $N_R$, $\kappa$, $n_u$, $c_u$ and $C_R$.
\item Apply convention 1 or 2 if relevant; otherwise compute $p$ and $s$.
\item Compute $\HD$, $\HC$ and $\KCDI_\lambda$.
\item Compute $A_u$, $r_u$ and $S_u$; then $\Abar$, $P$, $P_S$ and, if $\Abar>0$, $Q$.
\item Compute $\KJI_\lambda$. Assert, up to numerical tolerance, $P\leq\min(\Abar,P_S)$, $P\geq\Abar+P_S-1$ and $\KJI_\lambda\leq\KCDI_\lambda$.
\end{enumerate}
The assertions in step 5 are invariants of the definitions (Theorems~\ref{thm:frechet} and~\ref{thm:kji}), not empirical findings. They serve as regression tests for implementations.

% =====================================================================
\section{Scope}\label{sec:scope}
% =====================================================================

The results characterise a finite-group functional of two distributions. They do not justify the choice of groups, the uniform reference or the use of citation weight as a proxy for recognition. Each of these is a declared measurement decision, and altering it defines a different specification. The components discard document-level citation structure. By Proposition~\ref{prop:rate}, recognition alignment reduces to a comparison of group means, so it inherits the sensitivity of means to extreme values. KJI is a conjunctive composite of dependent components, not a decomposition into independent sources. Its interpretation, including any substantive reading of differences between corpora, belongs to empirical work such as \citet{gonzalezbarbado2026kratos}.

% =====================================================================
\appendix
\section{Numerical verification}\label{app:verify}
% =====================================================================

The ancillary file \texttt{anc/verify\_kratos\_properties.py} (Python, NumPy) checks the stated results. The checks cover:
\begin{enumerate}
\item every inequality and identity of Sections~\ref{sec:kcdi}--\ref{sec:kji}, evaluated on $20{,}000$ random configurations for each $G\in\{2,3,4,6,10\}$, with Dirichlet draws of varying concentration and randomly emptied groups;
\item the counterexample of Proposition~\ref{prop:nonmonotone};
\item the moments of Propositions~\ref{prop:null-global} and~\ref{prop:null-strat}, against $40{,}000$ Monte Carlo permutations of heavy-tailed (Pareto) weights.
\end{enumerate}
No violation of any deterministic result was found. The simulated null moments agree with the closed forms up to Monte Carlo error; Table~\ref{tab:null} reports the global case. The script uses the master seed $20260831$.

\begin{table}[h]
\centering
\caption{Global permutation null: simulated and closed-form $\Var(r_u)$ ($N_R=200$, $G=4$, Pareto weights, $B=40{,}000$).}\label{tab:null}
\begin{tabular}{lcccc}
\toprule
 & $u=1$ & $u=2$ & $u=3$ & $u=4$\\
\midrule
Simulation & 0.1734 & 0.2364 & 0.3874 & 0.3279\\
Proposition~\ref{prop:null-global} & 0.1735 & 0.2363 & 0.3851 & 0.3288\\
\bottomrule
\end{tabular}
\end{table}

\section*{Code availability}
The reference implementation of KRATOS v2.3 is archived at \url{https://doi.org/10.5281/zenodo.23082982} (concept DOI \url{https://doi.org/10.5281/zenodo.22749724}) and developed at \url{https://github.com/missmew33/kratos-streamlit-app}.

\section*{Author contributions}
M.\,D.\ González Barbado: conceptualisation, formal analysis, writing (original draft, review and editing). A.\ Barbado González: formal analysis, software, verification, writing (review and editing).

\bibliographystyle{plainnat}
\bibliography{kratos_math}
\end{document}